\documentclass{article}
\usepackage{ijcai22}

\usepackage{times}

\usepackage{soul}
\usepackage{url}
\usepackage[hidelinks]{hyperref}
\usepackage[utf8]{inputenc}
\usepackage[small]{caption}
\usepackage{graphicx}
\usepackage{amsmath}
\usepackage{booktabs}
\usepackage{amsthm}
\usepackage{microtype}
\usepackage{subfigure}
\usepackage{booktabs} % for professional tables
\usepackage{algorithmic}
\usepackage[ruled,vlined]{algorithm2e}
\usepackage{url}%ME
\usepackage{graphicx} %ME
\usepackage{graphics} %ME
\usepackage{subfigure} %ME
\usepackage{booktabs} % MEfor professional tables
\usepackage{multirow} %ME
\usepackage{mathrsfs} %ME
\usepackage{amsmath} %ME
\usepackage{amssymb} %ME
\usepackage{listings} % ME
\usepackage{color}
\newtheorem{definition}{Definition}

\newtheorem{proposition}{Proposition}

\title{\textsc{Cluster Attack}: Query-based Adversarial Attacks on Graphs\\ with Graph-Dependent Priors}

\author{
Zhengyi Wang$^{1,3}$
\and
Zhongkai Hao$^{1}$\and
Ziqiao Wang$^{1}$\and
Hang Su$^{\ast 1,2,3}$\And
Jun Zhu\thanks{Corresponding author.} $^{1,2,3}$
\affiliations
 $^{1}$Department of Computer Science \& Technology, Institute for AI, BNRist Center\\
Tsinghua-Bosch Joint ML Center, THBI Lab, Tsinghua University\\
$^{2}$Peng Cheng Laboratory\\
$^{3}$ Tsinghua University-China Mobile
Communications Group Co., Ltd. Joint Institute\\
\emails
\{wang-zy21, hzj21, ziqiao-w20\}@mails.tsinghua.edu.cn,
\{suhangss, dcszj\}@tsinghua.edu.cn
}

\begin{document}

\maketitle

\begin{abstract}
While deep neural networks have achieved great success in graph analysis, recent work has shown that they are vulnerable to adversarial attacks. Compared with adversarial attacks on image classification, performing adversarial attacks on graphs is more challenging because of the discrete and non-differential nature of the adjacent matrix for a graph. In this work, we propose Cluster Attack --- a Graph Injection Attack (GIA) on node classification, which injects fake nodes into the original graph to 
degenerate the performance of graph neural networks (GNNs) on certain victim nodes while affecting the other nodes as little as possible. We demonstrate that a GIA problem can be equivalently formulated as a graph clustering problem; thus, the discrete optimization problem of the adjacency matrix can be solved in the context of graph clustering. In particular, we propose to measure the similarity between victim nodes by a metric of \emph{Adversarial Vulnerability}, which is related to how the victim nodes will be affected by the injected fake node, and to cluster the victim nodes accordingly. Our attack is performed in a practical and unnoticeable query-based black-box manner with only a few nodes on the graphs that can be accessed. Theoretical analysis and extensive experiments demonstrate the effectiveness of our method by fooling the node classifiers with only a small number of queries. 
\end{abstract}

\section{Introduction}
Graph neural networks (GNNs) have obtained promising performance in various applications to graph data~\cite{ying2018graph,qiu2018deepinf,chen2019semi}. Recent studies have
shown that GNNs, like other types of deep learning models, are also vulnerable to adversarial attacks~\cite{dai2018adversarial,Z_gner_2018}. However, there still exists a gap between most of the existing attack setups and practice where the capability of an attacker is limited. Instead of directly modifying the original graph (aka., Graph Modification Attack, GMA)~\cite{yang2021derivative}, we focus on a more practical setting to inject extra fake nodes into the original graph (aka., Graph Injection Attack, GIA)~\cite{tdgia}. We perform query-based attack, which indicates that the attacker has no knowledge on the victim model but can access the model with a limited number of queries. 
As an example in Figure~\ref{fig: Fake Node Attack}, high-quality users in a social network may be misclassified as low-quality users after being connected with a fraudulent user whose features (maybe the meta data of the account) are carefully crafted utilizing query information from the target model. 

%\hangx{Maybe here you should present node classification? since the task for graph analysis is more complicated including node classification, link prediction, graph classfication, etc. }

\begin{figure}[t]
	\centering
	\includegraphics[width=8.5cm]{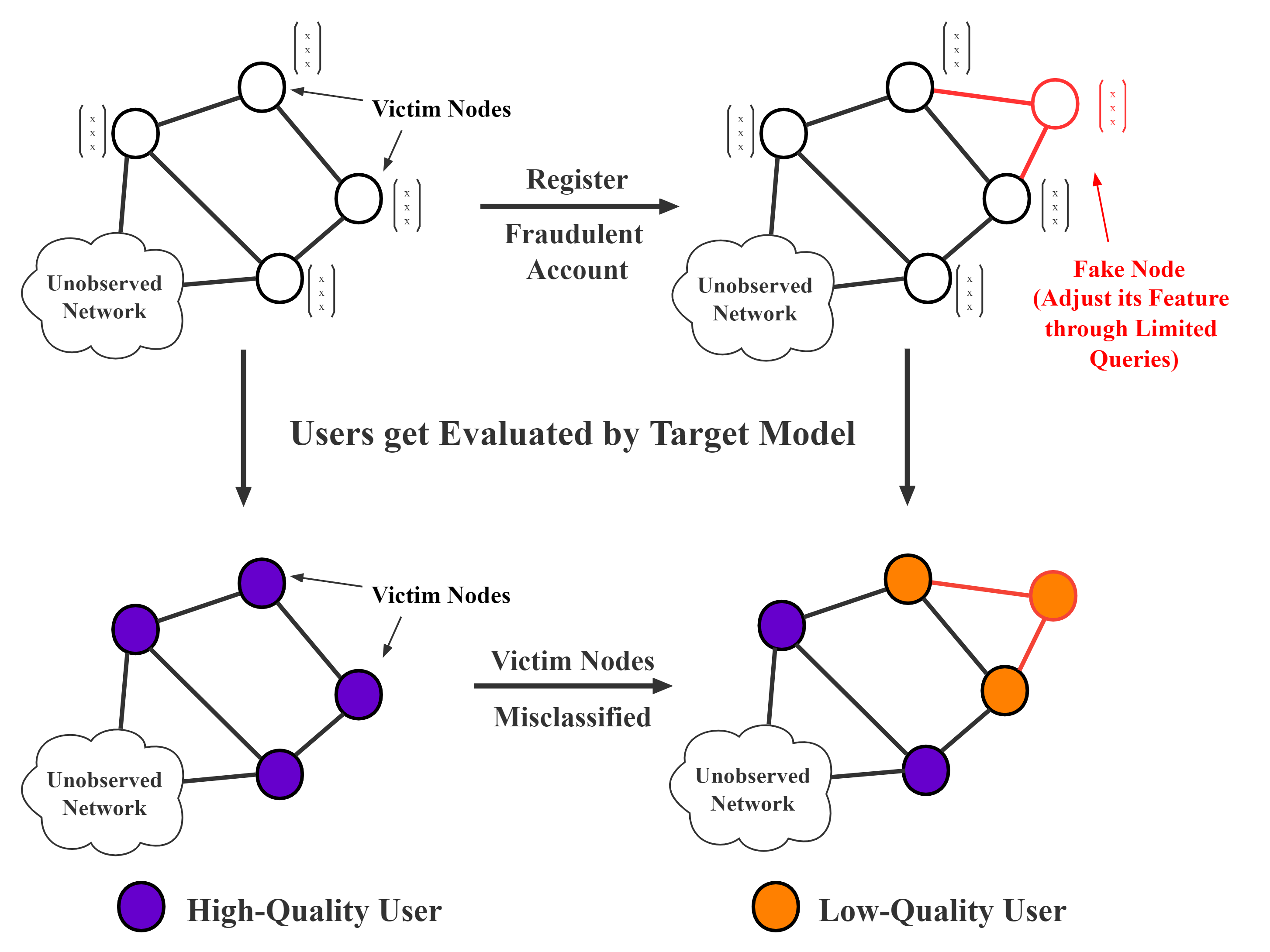}
	\caption{An illustration of query-based graph injection attack with partial information.% \hangx{no reference of the figure in this section} %\hangx{Maybe we can use a more informative figure to demonstrate our problem, i.e., the partial information, query-based, fake nodes, etc. I think we may use social network for the example to demonstrate the importance of the problem?}
	}
	%\vspace{-2ex}
	\label{fig: Fake Node Attack}
\end{figure}

 %\hangx{I think we may use an example to demonstrate the setting and importance of our problem, e.g., to describe a task on adversarial attacks on social networks for financial scenario? } 

Compared to adversarial attacks on image classification~\cite{10.1145/3128572.3140448,ilyas2018blackbox,Dong-2018}, the study of query-based adversarial attacks on graph data is still at an early stage. %\hangx{a quick jump to query-based attack, and explain it first, could use figure 1 by redrawing it } \hangx{moreover, why query-based attacks on graph is also important}
Existing attempts include training a surrogate model using query results \cite{wan2021adversarial}, a Reinforcement-Learning-based method~\cite{10.1145/3447548.3467416}, derivative-free optimization \cite{yang2021derivative} and a gradient-based method~\cite{10.1145/3460120.3484796}. However, most of the existing attacks simply adopt the optimization methods from other fields, such as image adversarial attacks, without utilizing the rich structure of graph data, which has considerable potential to achieve a higher performance attack.

Unlike most previous work, we only allow the adversary to access the information of a small part of the nodes since it is usually impossible to observe the whole graph, especially for large networks in practical scenarios. Moreover, we perform a \emph{black-box} attack, which does not allow the adversary to have access to the model structures or parameters. The adversary has only a limited number of queries on the victim model about the predicted scores of certain nodes, which is more practical. 
It is also noted that, owing to the non-i.i.d. nature of graph data, connecting victim nodes to fake nodes may have side effects on the accuracy of victim nodes' neighbors, which is not our purpose. To our best knowledge, this is the first attempt to limit the influence to a certain range of victim nodes and protect the other nodes from being misclassified simultaneously.

We propose a unified framework for query-based adversarial attacks on graphs, which subsumes existing methods. In general, the attacker decides on the current perturbation as a conditional distribution on history query results and current graph status. %\hangx{As one of the major contribution, you should say more about the framework}
Under this framework, we propose a novel attack method named \textsl{Cluster Attack}, which considers the graph-dependent priors by better utilizing the unique structure of the graph. In particular, 
% \hangx{should be more specific about what prior}.% \hangx{maybe we should clarify the setting first and then present the related works. Or exchange the order of paragraph 2 and 3 }
%The challenges of graph adversarial attacks lie in the discrete and non-differential nature of the adjacent matrix which makes it hard to directly apply the existing adversarial attack method based on gradients.%\hangx{this challenge should be moved to somewhere before}
we try to find an equivalent discrete optimization problem. We first demonstrate that a GIA problem can be formulated as an equivalent graph clustering problem. Because the discrete optimization problem of an adjacent matrix can be solved in the context of graph clustering, we prevent query-inefficient searching in the non-Euclidean space. The resulting cluster serves as a graph-dependent prior for the adjacent matrix, which utilizes the vulnerability of the local structure. %\wzy{TODO}. \hangx{I think you why the graph clustering is a good prior for the optimization is still not clear. This the key point in our paper, and you should take efforts to explain it.}
Second, the key challenge in graph clustering is to define the similarity metric between nodes. We propose a metric to measure the similarity of victim nodes, called Adversarial Vulnerability; this is related to how the victim nodes will be affected by the injected fake node, and we cluster the victim nodes accordingly. The Adversarial Vulnerability is only related to the local structure of the graph and thus can be handled with only part of the graph observed. %\hangx{the metric aims to address what challenge?}
%   \hangx{how you address the partial information? how you use the query information? how we use the graph structure prior and how the prior facilitates the query-based attacks? I think we should reorganize the section again by reflecting how you address the technical challenges} 

% \hangx{we should focus on the query-based attacks here. And 1) provide a summary of the present methods expecially for the query-based attacks 2) what are the limitations of the current methods  }

% The unreasonable assumptions include that some attackers can alter a large proportion of nodes, or the attacker may have full knowledge about the underling GNN model for the victims. Some works only evaluate the performance on the victim nodes without considering the side effects on other graph node, which will make attacks more detectable. 

Our contributions are summarized as follows:
\begin{itemize}
\item We propose a unified framework for query-based adversarial attacks on graphs, which formulates the current perturbation as a conditional distribution on the history of query results and on the current graph status.
%, and is consistent with existing methods.%\hangx{provide the summary of the framework and how it can unified the previous works}% Under our framework, we propose a query-based method with graph-dependent prior under a practical threat model. \hangx{move the novel method to the second contribution}
% \item We propose a practical threat model on graph adversarial attack. We perform query-based adversarial attack on graph with partial information about the graph with non-targeted nodes protected.
\item We propose \textsl{Cluster Attack}, an injection adversarial attack on a graph, which formulates a GIA problem as an equivalent graph clustering problem and thus solves the discrete optimization of an adjacent matrix in the context of clustering.
\item After providing theoretical bounds on our method, we empirically show that our method achieves high performance in terms of success rate of attacks under an extremely strict setting with a limited number of queries and only part of the graph observed.% \hangx{we should highlight the main contribution in the experiments, e.g., achieve the adversarial attacks with limited number of queries?}
\end{itemize}
% \hangx{rephrase it. 1) unified framework and some analysis 2) a novel method which is much more efficient 3) main theoretical and experimental results}

\section{Background}
In this section, we present recent works on node classification and adversarial attacks on graphs.

\subsection{Node Classification on a Graph}
% \junz{in Intro, only mentioned GNN, why switch to GCN? node classification is a task that can be implemented via GNN or GCN.} \wzy{This is because GCN is one of the most representative GNNs. I have added some words.}
Node classification on graphs is an important task, with a wide range of applications such as user classification in financial networks. %\hangx{maybe refer to some examples on why node classification is important, better use some practical safety-critical examples } 
It aims to carry out classification by aggregating the information from neighboring nodes~\cite{kipf2016semisupervised,hamilton2017inductive,velikovi2017graph}. Recent work has carried out node classification using graph convolutional networks (GCNs)~\cite{kipf2016semisupervised}, which is one of the most representative GNNs. Specifically, let a graph be $G=(\mathbf{A},\mathbf{X})$, where $\mathbf{A}$ and $\mathbf{X}$ respectively represent the adjacency matrix and the feature matrix. Given a subset of labeled nodes in the graph, GCN aims to predict the labels of the remaining unlabeled nodes in the graph as
\begin{equation}
    f(G)=f(A,X)=\mbox{softmax}\left( \hat{\mathbf{A}}\sigma(\hat{\mathbf{A}}\mathbf{X}\mathbf{W}^{(0)})\mathbf{W}^{(1)} \right),\label{GCN}
\end{equation}
where $\hat{\mathbf{A}}$ is the normalized adjacency matrix; $\mathbf{W}^{(0)}$ and $\mathbf{W}^{(1)}$ are parameter matrices; $\sigma$ is the activation function; and $f(G)$ is the prediction corresponding to each node.
\subsection{Graph Adversarial Attacks}

Numerous methods have been developed to perform adversarial attacks on graphs. Early works focused on modifying the original graph (i.e., Graph Modification Attack) \cite{dai2018adversarial,Z_gner_2018}, while some recent works \cite{tdgia,gnia} have focused on a more practical setting to inject extra fake nodes into the original graph (i.e., Graph Injection Attack). For query-based graph adversarial attacks, as shown in Table \ref{unified}, various efforts have been made. Some have focused on attacking the task of node classification \cite{yang2021derivative} while there also have been efforts to attack graph classification \cite{wan2021adversarial,10.1145/3447548.3467416,10.1145/3460120.3484796}, with gradient-based \cite{10.1145/3460120.3484796} or gradient-free methods \cite{yang2021derivative,wan2021adversarial,10.1145/3447548.3467416}. Nevertheless, most of the existing attacks just adopt optimization methods from other fields (especially image adversarial attack), ignoring the unique structure of graph data. In this work, we propose to attack in a graph-specific manner utilizing the inherent structure of a graph. %To our best knowledge, we are the first to perform query-based black-box graph injection attack on node classification. \hangx{related work is to position our work. if the key novelty is the new setting, and why this setting is important and challenging?}

\section{A Unified Framework for Query-Based Adversarial Attacks on Graphs}
\begin{table*}[tbhp]
	%\vspace{-1ex}
	\centering
	\small
	%\resizebox{\textwidth}{15mm}{
	\begin{tabular}{lll}
		\toprule
        Method & Optimization Step & Target Task \\
        \midrule
        Random&$\Delta G\sim \mbox{Random}$ & - \\
        \midrule
        GRABNEL  &\multirow{2}{*}{$\Delta G = \underset{\Delta G}{\mbox{argmin}}\,\mathcal{L}_{sur}(G+\Delta G)$}&  Graph Classification\\
        \cite{wan2021adversarial}&&(GMA+GIA)\\
        \midrule
        DFO \cite{yang2021derivative}&$\Delta G = F(G, \delta)-G$, $\delta \sim$ DFO & Node Classification (GMA)\\
        \midrule
        \cite{10.1145/3460120.3484796}&$\nabla p(\Delta A) = \frac{1}{Q}\sum_{1}^{Q}\mbox{sgn}(\frac{p(\Delta A +\mu u_q)-p(\Delta A)}{\mu}u_q)$, $u_q\sim$ Gaussian& Graph Classification (GMA)\\
        \midrule
        Rewatt \cite{10.1145/3447548.3467416} &$\Delta G\sim p(\cdot|G)$ from Reinforcement Learning Agent & Graph Classification (GMA)\\
        \midrule
        \multirow{2}{*}{Cluster Attack (Ours)} &$\left\{
        \begin{array}{l}
              \Delta \mathbf{X}_{fake} = \mathbb{I}(\mathcal{L}(\mathbf{A}^+, \mathbf{X}^+)>\mathcal{L}(\mathbf{A}^+, \begin{bmatrix}\mathbf{X}\\\mathbf{X}_{fake}+\delta \mathbf{X}_{ij}\end{bmatrix}))\cdot\delta \mathbf{X}_{ij}\\
            \nabla_{\mathbf{X}_{fake}}\mathbb{E}\,\mathcal{L}(\mathbf{A}^+,\mathbf{X}^+) = \frac{1}{\sigma n}\sum_{i=1}^{n}Z_i\mathcal{L}(\mathbf{A}^+,\begin{bmatrix}\mathbf{X}\\\mathbf{X}_{fake}+\sigma Z_i\end{bmatrix})
        \end{array}
        \right.$& \multirow{2}{*}{Node Classification (GIA)} \\
        &$\Delta A\sim$ cluster prior&\\
        % Cluster Attack (with Bandit) &$\Delta G = \mbox{sgn}^{*}(\mathcal{L}(G)-\mathcal{L}(G+\delta G))\cdot\delta G$, $\delta G_{(t)}\sim \phi_t(\delta G)$ &  \\
		\bottomrule
	\end{tabular}
	%\vspace{-3ex}
	%}
	\caption{Existing query-based methods on graph adversarial attacks.}
	\label{unified}
\end{table*}

We now present a unified framework for query-based adversarial attacks as well as the threat model and loss function.

\subsection{Graph Injection Attack}
Given a small set of victim nodes $\Phi_{\mathbf{A}}\subseteq \Phi$ in the graph, the goal of graph injection attack is to perform mild perturbations on the graph $G=(\mathbf{A},\mathbf{X})$, leading to $G^{+}=(\mathbf{A}^{+},\mathbf{X}^{+})$, such that the predicted labels of the victim nodes in $\Phi_{\mathbf{A}}$ are changed into the target labels. This goal is usually achieved by optimizing the adversarial loss $\mathcal{L}(\cdot)$ under constraints as:
\begin{equation}
    \min_{G^{+}}\, \mathcal{L}(G^{+})\,\,s.t.\,\,\mbox{dist}(G, G^{+})\leq \Delta\label{overall0},
\end{equation}
where $\mbox{dist}(G, G^{+})$ denotes the magnitude of perturbation and has to be within the adversarial budget $\Delta$. In this section, it can be specified as $|\Phi_{fake}|\leq \Delta_{fake}$ and $\sum_{u\in \Phi_{fake}} d(u) \leq \Delta_{edge}$ with $d(u)$ being the degree of node $u$.

% Minimizing $\mathcal{L}(\cdot)$
%\hangx{you should have some comments on each equation, e.g., the meaning of the objective and the constraint.}

In graph injection attacks, we have the augmented adjacent matrix $\mathbf{A}^{+}=\begin{bmatrix}\mathbf{A}&\mathbf{B}^{T}\\\mathbf{B}&\mathbf{A}_{fake}\end{bmatrix}$ and the augmented feature matrix  $\mathbf{X}^{+}=\begin{bmatrix}\mathbf{X}\\\mathbf{X}_{fake}\end{bmatrix}$. We further use $\Phi^{+}=\Phi\cup \Phi_{fake}$ to denote the node set of $G^{+}$. In particular, $\mathbf{X}_{fake}$ corresponds to the feature of fake nodes; $\mathbf{B}$ denotes the connections between fake nodes and original nodes, and $\mathbf{A}_{fake}$ denotes the mutual connections between fake nodes. The malicious attacker manipulates $\mathbf{A}_{fake},\mathbf{B}$ and $\mathbf{X}_{fake}$, leading to as low classification accuracy on $\Phi_{\mathbf{A}}$ as possible. 
% In our case, $\mbox{dist}(G, G^{+})\leq \Delta$ represents
% \begin{equation}
%     N_{fake}\triangleq|\Phi_{fake}|\leq \Delta_{fake},\,\sum_{u\in \Phi_{fake}} d(u) \leq \Delta_{edge},
% \end{equation}
% where $d(u)$ denotes the degree of node $u$.
%\hangx{no definition on $d(u)$}
%\subsubsection{A Unified Formulation for Query-Based Graph Adversarial Attack}

In a query-based adversarial attack on graphs, we uniformly formulate the update of a graph at time $t$ as
\begin{equation}
    \Delta G_{(t)}\sim  p\big(\Delta G_{(t)}|\{f(G_i)|i=1,2,...,q_t\},G_{t})\mbox{,}
\end{equation}
where $\{f(G_i)|i=1,2,...,q_t\}$ denotes the feedback (hard labels or predicted values) of total $q_t$ queries in the history from the target model. A curated list of current query-based graph adversarial attacks is shown in Table \ref{unified}. In general, perturbation $\Delta G_{(t)}$ in time step $t$ is conditioned on the history of query results $\{f(G_i)|i=1,2,...,q_t\}$ and the current graph status $G$. Previous work has focused on using different optimization methods, including reinforcement learning \cite{10.1145/3447548.3467416} and gradient-based optimization \cite{10.1145/3460120.3484796} to decide $\Delta G$ without utilizing the graph structure explicitly.%\wzy{}% \hangx{As a major contribution, I still suggest you to make more analysis about it, e.g., how to incorporate the prior, and how the previous can be formulated using the framework}
%\hangx{the table is important, and I think you should provide more analysis on the table.}

\subsection{Threat Model}
%Our attack is conducted in an unnoticeable and practical manner.
\paragraph{Adversary Capability.} We greatly restrict the attacker's ability so that we can only make connections between victim nodes and fake nodes. No connections can be made between fake nodes because connected malicious fake nodes are easier for detectors to locate. The number of new edges $\Delta_{edge}$ is set as barely the number of victim nodes $\Phi_A$, which means that each victim node is connected by only one new edge.
\paragraph{Protected Nodes.} As mentioned above, owing to the non-i.i.d nature of graph data, attacking victim nodes may have unintended side effects on their neighboring nodes. While attacking victim nodes, we simultaneously aim to keep the labels of untargeted nodes unchanged, to make our perturbation unnoticeable. In our setting, we try to protect $\mathcal{N}_{k}(\Phi_{A})$, which are the neighbors of the victim nodes within $k$-hop.

\paragraph{Partial Information.} It is practical to assume that the attacker has access to only part of the graph when conducting the attack. As mentioned above, we adopt an extremely strict setting so that we can only observe the features and connections of the observed nodes defined as
\begin{equation}
    \Phi_o=\Phi_A\cup\mathcal{N}_{k}(\Phi_{A})\cup\Phi_{fake}\label{observed_nodes}.
\end{equation}

\paragraph{Limited Queries.} It is practical in real scenarios that we have a limited number of queries to the victim model rather than full outputs of arbitrarily many chosen inputs. In our setting, we can query at most $K$ times in total for the predicted scores of the observed nodes. The architecture and parameters about the victim model are unknown to the attacker.

\begin{figure*}[t]
	\centering
	\includegraphics[width=\linewidth]{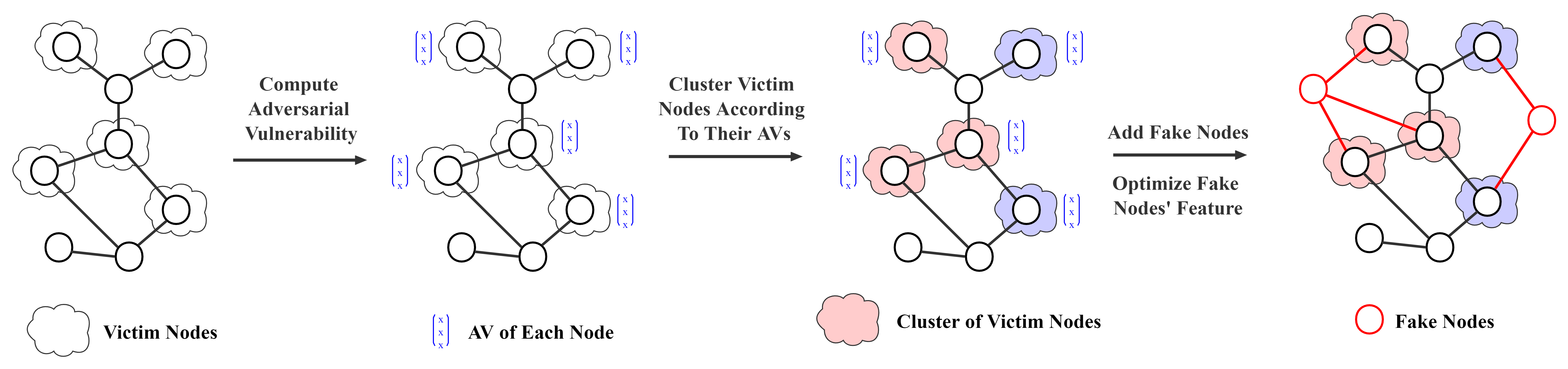}
	\caption{An illustration of Cluster Attack. We first compute Adversarial Vulnerability for each victim node with a limited number of queries; after that, we cluster the victim nodes and inject fake nodes accordingly; and finally we optimize the fake nodes' features.}
		%\vspace{-2ex}
	\label{fig: Cluster Attack}
\end{figure*}

\subsection{Loss Function}
We aim to make the classifier misclassify as many nodes as possible in the victim set of $\Phi_{\mathbf{A}}$ . As it is nontrivial to directly optimize the number of misclassified nodes since the objective is discrete, we choose to optimize a surrogate loss function:
\begin{align}
\min_{G^{+}}&\, \mathcal{L}(G^{+})\triangleq \sum_{v\in \Phi_{\mathbf{A}}}{\ell(G^{+},v)}+\lambda\sum_{v\in \mathcal{N}_{k}(\Phi_{\mathbf{A}})}{\ell_{\mathcal{N}}(G^{+},v)},\label{overall_goal}\notag\\
& s.t.\,\,\mbox{dist}(G, G^{+})\leq \Delta,
\end{align}
%$$s.t.\,\,\mbox{dist}(G, G^{+})\leq \Delta,$$
where $\ell(G^{+},v)$ and $\ell_{\mathcal{N}}(G^{+},v)$ represent the loss functions for each victim node and for a protected node, respectively. A smaller $\ell(G^{+},v)$ means that node $v$ is more likely to be misclassified by the victim model $f$; by contrast, a smaller $\ell_{\mathcal{N}}(G^{+},v)$ means that the predicted label of node $v$ is less likely to be changed during our attack. In particular, we design our loss in the manner of the C\&W loss \cite{carlini2016evaluating}, and define: %\hangx{I still suggest you to include the detailed form of the loss function in the main text, and better explicit with $A$ and $X$, since you need to explain the challenges for optimization in terms of $A$}
% for $\ell$ and $\ell_{\mathcal{N}}$.\hangx{more specific about the loss. Why we use the similar C\&W loss but the function are different for $\ell$ and $\ell_N$ }
%Specifically, we have
\begin{equation}
\begin{aligned}
    &\ell(G^{+},v) = \\
    &\sigma\left( \max_{y_i\neq y_t} ([f(\mathbf{A}^{+},\mathbf{X}^{+})]_{v,y_i})-[f(\mathbf{A}^{+},\mathbf{X}^{+})]_{v,y_t} \right),    \label{l1}
\end{aligned}
\end{equation}where $y_t$ stands for the target label of node $v$ and the attacker succeeds only when node $v$ is misclassified as $y_t$. $\sigma(x)=\max(x,0)$. $[f(G^{+})]_{v,y_i}$ denotes the output logit of node $v$ of class $y_i$. For protected nodes, we define:
\begin{equation}
\begin{aligned}
&\ell_{\mathcal{N}}(G^{+},v) =\\
& \sigma\left( \max_{y_i\neq y_g}  [f(\mathbf{A}^{+},\mathbf{X}^{+})]_{v,y_i})-[f(\mathbf{A}^{+},\mathbf{X}^{+})]_{v,y_g} \right)\mbox{,}\label{l2}
\end{aligned}
\end{equation}
where $y_g$ is the ground-truth label of $v$ from the victim model.
%Overall loss $\mathcal{L}(G^{+};\Phi_{A})$ is summed over the loss of each victim node along with the loss of each protected neighboring node with a trade-off parameter $\lambda$. In practice, we add square root to the loss of each victim node to favor the nodes which are likely to be misclassified.

% Optimization of Eq. (\ref{overall_goal}) is challenging due to the discrete nature of $G^{+}$ and large space of possible choices of $G^{+}$. To tackle the optimization problem, we propose our Cluster Attack.

\section{Cluster Attack with Graph-Dependent Priors}
\subsection{Cluster Attack}
%\hangx{Maybe you should consider how you use the prior}

The combinatorial optimization problem \eqref{overall_goal} is hard to solve owing to the non-Euclidean nature of the adjacent matrix and the complex structure of neural networks. To tackle this, we tried to find an equivalent combinatorial optimization problem and transform our GIA problem into a well-studied one. %\hangx{since you do not have explicit form of $A$ in Eq. (6), and the presentation is not smooth here} 
Here, we point out that every choice of adjacent matrix has an equivalent representation of a division of victim nodes into clusters. Thus, we get our key insight that this discrete optimization problem can be transformed into an equivalent graph clustering problem, %\junz{refer to some classical book/papers}, 
which is a well-studied discrete optimization problem~\cite{schaeffer2007graph}.% \hangx{I think it is really a surprise here, and how you derive to the graph clustering?}
% \hangx{a jump of graph cluster. More analysis to derive the problem.}

\begin{proposition}[GIA/Graph Clustering Equivalence] Given graph $G$ and a set of nodes $\Phi_A \subseteq \Phi$, for a division of the victim nodes $\Phi_{\mathbf{A}}$ into $N_{fake}$ clusters $C=\{C_{1},C_{2},...,C_{N_{fake}}\}$, $\cup_{C_{i}\in C}C_{i}=\Phi_{\mathbf{A}}$, there exists a corresponding $\mathbf{B}$ and vice versa.\label{equivalence}
\end{proposition}
\begin{proof}We provide a one-to-one mapping between cluster $C$ and adjacent matrix $\mathbf{B}$. Specifically, given $C=\{C_{1},C_{2},...,C_{N_{fake}}\}$, $\cup_{C_{i}\in C}C_{i}=\Phi_{\mathbf{A}}$ we get $\mathbf{B}$ from
\begin{equation}
\mathbf{B}_{ij}=\begin{cases} 
1,  & \mbox{if }v_j\in C_i\\
0, & \mbox{otherwise}\label{B_equation}
\end{cases}.
\end{equation}
We have
\begin{equation}
    \Delta_{edge} = \sum_i\sum_j\mathbf{B}_{ij}=\sum_i|C_i|=|\Phi_A|.
\end{equation}
Thus, we resultant $\mathbf{B}$ is valid in our setting.

Given $\mathbf{B}$ where $|\Phi_A| = \Delta_{edge} = \sum_i\sum_j\mathbf{B}_{ij}$, we derive the cluster $C$ from
\begin{equation}
    v_j\begin{cases} 
    \in C_i,  & \mbox{if  }\mathbf{B}_{ij}=1\\
    \notin C_i, & \mbox{if  }\mathbf{B}_{ij}=0\label{C_equation}
    \end{cases}.
\end{equation}
We have $\sum_i|C_i|=\sum_i\sum_j\mathbf{B}_{ij} = |\Phi_A|$.

In our setting, each victim node is connected to only one fake node, which indicates
\begin{equation}
    \sum_i \mathbf{B}_{ij} = 1,\,\,\forall\,v_j\in\Phi_A.
\end{equation}
In this case, each cluster gets disjoint with each other
\begin{equation}
    C_i\cap C_j=\emptyset,\,\,\forall\,1\leq i,j\leq N_{fake}.
\end{equation}
Then we get $\cup_{C_{i}\in C}C_{i}=\Phi_{\mathbf{A}}$ and $C$ is a valid division.
\end{proof}
Because $\textbf{A}_{fake}=\textbf{0}$ is fixed in our setting, we formulate our graph injection attack problem as an equivalent graph clustering problem. As a result, the non-trivial discrete optimization problem of the adjacency matrix can be solved in the context of graph clustering. The resulting cluster serves as a graph-dependent prior for adjacent matrix $\mathbf{B}$, which prevents inefficient searching in non-Euclidean discrete space.%\hangx{this is one of our major contribution, and we may provide more analysis to demonstrate the equivalence }

For graph clustering, the main concern is the metric of the similarity between victim nodes. To investigate how a fake node will affect the performance on a certain node, we propose Adversarial Vulnerability as the similarity metric for graph clustering.
Adversarial Vulnerability of a victim node reflects its ``most vulnerable angle'' towards adversarial features of fake nodes, which is related only to the local structure of the graph and can be handled with only part of the graph observed. We have the insight that victim nodes sharing similar Adversarial Vulnerability are more likely to be affected simultaneously when they are connected to the same fake node.% \hangx{another jump to adversarial vulnerability. what is the connection with the graph cluster?}

\begin{definition}[Adversarial Vulnerability] For victim node $v\in \Phi$, its \textsl{Adversarial Vulnerability} is defined as
\begin{equation}
\mbox{AV}(v) = \underset{x_{u}}{\mbox{argmin}}\,\,\mathcal{L}(G^{+}),\label{av}
\end{equation}
where $x_u$ denotes the feature of fake node $u$ connected to node $v$. For the fake node itself, the Adversarial Vulnerability is defined as its own feature.
\end{definition}
Here, we adopt Euclid's distance as distance metric between victim nodes' Adversarial Vulnerability which is in Euclidean feature space.%\hangx{is it a good choice to use Euclid distance since we have emphasize the non-euclid space previously. need to be further justified.}
% \hangx{Again, why we need this metric.}

\begin{definition}[Adversarial Distance Metric] $\forall v,u\in \Phi^+$, the \textsl{Adversarial Distance Metric} between node $v$ and $u$ is defined as $d(v,u)=||\mbox{AV}(v)-\mbox{AV}(u)||_{2}^{2}$.
\end{definition}
After the Adversarial Vulnerability is computed, the objective of the cluster algorithm is to minimize the following cluster distance as
\begin{equation}
\min_{C}\sum_{C_{i}\in C}\sum_{v\in C_{i}} d(v,c_{i})\label{cluster_loss},
\end{equation}
where the cluster center $c_i$ is the corresponding fake node of cluster $C_i$, and
\begin{equation}
    \mbox{AV}(c_{i})=\frac{1}{|C_{i}|}\sum_{v\in C_{i}}\mbox{AV}(v).\label{cluster_center}
\end{equation}

\subsection{Optimization}

%\hangx{Some important technical issues are missing. actually, I do not see how to use the priors and the cluster information. should be more explicit ; moreover, you should also explain how to address the partial information, and how to protect the untarget nodes in the algorithm.}

%Figure \ref{fig: Cluster Attack} is an overview of our attack algorithm. We first compute Adversarial Vulnerability for each victim node within limited queries. Then we cluster the victim nodes and inject fake nodes accordingly.%\hangx{I think you may explain the key idea and big picture using the figure first} 
To approximate Adversarial Vulnerability, we adopt zeroth-order optimization~\cite{10.1145/3128572.3140448}, which is similar to query-based attacks on image classification, to better utilize limited queries. For graph with discrete features, the optimization of Eq. (\ref{av}) can be 
\begin{equation}
\begin{aligned}
\label{discrete}&\Delta \mathbf{X}_{fake} = \\
&\mathbb{I}\left( \mathcal{L}(\mathbf{A}^+, \mathbf{X}^+)>\mathcal{L}\left( \mathbf{A}^+, \begin{bmatrix}\mathbf{X}\\\mathbf{X}_{fake}+\delta \mathbf{X}_{ij}\end{bmatrix} \right) \right)\cdot\delta \mathbf{X}_{ij}\mbox{,}
\end{aligned}
\end{equation}
where $\delta \mathbf{X}_{ij}$ denotes the tentative perturbation in dimension $j$ of a feature of the $i$th fake node and $\mathbb{I}(\cdot)$ is the indicator function. A tentative perturbation is adopted only if it diminishes the adversarial loss. For continuous feature space, we adopt NES~\cite{ilyas2018blackbox} for gradient estimation as 
\begin{equation}
\begin{aligned}
    \label{continuous}&\nabla_{\mathbf{X}_{fake}}\mathbb{E}\,\mathcal{L}(\mathbf{A}^+,\mathbf{X}^+) = \\ &\frac{1}{\sigma n}\sum_{i=1}^{n}Z_i\mathcal{L}\left( \mathbf{A}^+,\begin{bmatrix}\mathbf{X}\\\mathbf{X}_{fake}+\sigma Z_i\end{bmatrix} \right)\mbox{,}     
\end{aligned}
\end{equation}
where $\sigma>0$ is the standard variance, $n$ is the size of the NES population and $Z_i\sim \mathcal{N}(\mathbf{0},\mathbf{I}_{N_{fake}\times D})$ is the perturbation of $\mathbf{X}_{fake}$. After gradient estimation, gradient-based optimization methods can be adopted. Here, we use Projected Gradient Descent (PGD) \cite{madry2017towards} to optimize $\mathbf{X}_{fake}$.

Our method is outlined in Figure \ref{fig: Cluster Attack}. With the resultant Adversarial Vulnerability, we solve the optimization of Eq. (\ref{cluster_loss}) by K-Means clustering. After that, the features of fake nodes, initialized as the cluster center in Eq.~\eqref{cluster_center}, are optimized using Eq. (\ref{discrete}) and Eq. (\ref{continuous}). More details of our algorithm are deferred to the appendix. %\hangx{polish the algorithm to reflect the main idea. Remember that algorithm in a paper is not a pseudocode but to reflect the main steps and key contributions }
% \hangx{you should explain the algorithm briefly} 

\subsection{Theoretical Analysis}

Connecting fake nodes to an original graph brings a victim node (1) one 1-hop neighbor (the fake node connected to it); (2) neighbors at a farther distance connected through this fake node; (3) fake nodes connected to other victim nodes which are at least 2-hop away. It is noted that 1-hop neighbors are often dominant. Here we leave out the influence of farther neighbors caused by the fake node and fake nodes connected to other victim nodes which are at least 2-hop or even farther. The loss function over the $i$th victim node $v_i\in \Phi_A$ in Eq. (\ref{overall_goal}) can thus be seen as a function of fake nodes' features (here we set a trade-off parameter $\lambda=0$ for analysis). We have
% To theoretically examine the effectiveness of our algorithm, we set trade-off parameter $\lambda=0$ in Eq. (\ref{overall_goal}) and the total loss becomes a sum of loss functions over each victim nodes. Connecting a fake node to a victim nodes bring the victim node one 1-hop neighbor (the fake node itself) and much more neighbors in further distance. However, 1-hop neighbors are often dominants. If we leave out the influence of further neighbors brought by the fake node and fake nodes connected to other victim nodes which are usually far away, the loss function over $i$th victim node $v_i\in \Phi_A$ can be seen as a function of fake nodes' feature. We have
\begin{equation}
    \mathcal{L}(G)=\sum_{v_i\in \Phi_A} l(G^{+},v_i)= \sum_{i=1}^{|\Phi_A|} l_i(x_i)\label{divide}\mbox{,}
\end{equation}
where $x_i$ denotes fake nodes' features connected to victim node $v_i$. Theoretically, we provide our bounds under certain smooth conditions which hold for numerous neural networks.
\begin{definition}[W-condition] We say that a function $\mathcal{L}(G)=\sum_{i=1}^{|\Phi_A|}l_i(x_i)$ satisfies the  \textsl{W-condition}, if and only if $\forall\,1\leq i\leq |\Phi_A|$, $l_i(\cdot)$ satisfies the Lipschitz condition of order 2. In this case, we have % and is strong convex respect to $x_i$ in neighborhood of minimum $x_{i}^{*}$.
\begin{equation}
    m_i||x_i-x_i^{*}||_2^2\leq l_i(x_i)-l_i(x_i^*)\leq M_i||x_i-x_i^{*}||_2^2\mbox{,}\label{wcondition}
\end{equation}
where $0\leq m_i\leq M_i$ are constants and $1\leq i\leq |\Phi_A|$. $x_{i}^{*}$ is the minimum of $l_i(\cdot)$.
\end{definition}
Note that $M$ exists because the loss function satisfies the Lipschitz condition of order 2 under W-condition; and it also includes $m$ because $m=0$ always satisfies Eq. (\ref{wcondition}). Under W-condition, we derive our bounds on the difference in adversarial loss between our results and optimal adversarial examples.
\begin{proposition}\label{pro:bound}
If $\mathcal{L}(\cdot)$ in Eq. (\ref{divide}) satisfies the W-condition, $G^{m}$ is the optimal choice of Eq. (\ref{overall0}) and $G^{'}$ is the optimal given by the Cluster Attack of Eq. (\ref{cluster_loss}). Then, we have
\begin{equation}
    \mathcal{L}(G^{'})-\mathcal{L}(G^{m})\leq |M-m|\min_{C}\sum_{C_{i}\in C}\sum_{v\in C_{i}} d(v,c_{i})\mbox{,}\label{bound}
\end{equation}
where $M=\max_{1\leq i\leq |\Phi_A|}M_i$ and $m=\min_{1\leq i\leq |\Phi_A|}m_i$.
\end{proposition}
% \paragraph{Proposition 3} Single layer GCN satisfies W-condition.
\begin{proof}
We have
\begin{equation}
\begin{aligned}
    &\mathcal{L}(G^{'})-\mathcal{L}(G^{m})\\ =&\sum_{i=1}^{|\Phi_A|}l_i(x_i')-\sum_{i=1}^{|\Phi_A|}l_i(x_i^m)\\
    =&\sum_{i=1}^{|\Phi_A|}(l_i(x_i')-l_i(x_i^*))-\sum_{i=1}^{|\Phi_A|}(l_i(x_i^m)-l_i(x_i^*))\\
    \leq& \sum_{i=1}^{|\Phi_A|}M_i ||x_i'-x_i^*||_2^2-\sum_{i=1}^{|\Phi_A|}m_i ||x_i^m-x_i^*||_2^2\\
    \leq& M\sum_{i=1}^{|\Phi_A|}||x_i'-x_i^*||_2^2-m\sum_{i=1}^{|\Phi_A|}||x_i^m-x_i^*||_2^2\\
    =&M\min_{C}\sum_{C_{i}\in C}\sum_{v\in C_{i}} d(v,c_{i})-m\sum_{i=1}^{|\Phi_A|}||x_i^m-x_i^*||_2^2\\
    \leq&M\min_{C}\sum_{C_{i}\in C}\sum_{v\in C_{i}} d(v,c_{i})-m\min_{C}\sum_{C_{i}\in C}\sum_{v\in C_{i}} d(v,c_{i})\\
    =&|M-m|\min_{C}\sum_{C_{i}\in C}\sum_{v\in C_{i}} d(v,c_{i}),
\end{aligned}
\end{equation}
where $x_i^{m}$ and $x_i'$ are features of the fake node connected to $i$th victim node provided by Eq. (\ref{overall0}) and Cluster Attack of Eq. (\ref{cluster_loss}), respectively.
\end{proof}
Proposition \ref{pro:bound} indicates that the difference in adversarial loss between our results and optimal adversarial examples is bounded by the minimal cluster distance in Eq. \eqref{cluster_loss} and how each $l_i(\cdot)$ is linear to $||x_i-x_i^*||_2^2$.% Proof of our propositions and more theoretical analysis are shown in the appendix.% \hangx{you should explain what does the proposition imply?}

% Our method solves the optimization problem in Eq. (\ref{overall_goal}) by minimize the object function Eq. (\ref{cluster_loss}), which can be easily solved by cluster algorithm. Our method prevents the time-consuming search of the large space of $\mathbf{A}_{fake},\mathbf{B},\mathbf{X}_{fake}$ and is thus an efficient method.

\section{Experiments}

\subsection{Experimental Setup}

\paragraph{Dataset.} We do our experiments on Cora and Citeseer \cite{Sen_Namata_Bilgic_Getoor_Galligher_Eliassi-Rad_2008}, which are two benchmark small citation networks with discrete node features, and on Reddit~\cite{hamilton2017inductive} and ogbn-arxiv~\cite{hu2020open}, which are two large networks with continuous node features. The statistics of the datasets are shown in Table \ref{dataset}.% Nodes in Cora and Citeseer, for example, are papers and the edges represent the relationship of citation. In this case, fake nodes and extra edges may be papers of low quality without peer review process and citations which may be a potential threat when analyzing academic data using machine learning models. 
\begin{table}[tb]
	%\vspace{-0.4em}
	\centering
	\begin{tabular}{lllll}
		\toprule
		Name       & Nodes  & Edges    &Features    &Classes \\
		\midrule
		Cora      & 2708   & 5429     &1433        &7       \\
		Citeseer   & 3327   & 4732     &3702        &6       \\
		Reddit     & 232965 & 11606919 &602         &41      \\
		ogbn-arxiv & 169343 & 1157799  &128         &40      \\
		\bottomrule
	\end{tabular}
	%\vspace{-0.2cm}
	\caption{Statistics of the datasets.}
	\label{dataset}
\end{table}

\begin{table*}[tbhp]
	\centering
	%\resizebox{\textwidth}{15mm}{
	\begin{tabular}{l|llll|llll}
		\toprule
		\multirow{2}{*}{Method} &
        \multicolumn{4}{c|}{Cora}&
        \multicolumn{4}{c}{Citeseer} \\
        & $T=3$ & $T=5$ & $T=7$ & $T=10$ &$T=3$ & $T=5$ & $T=7$ & $T=10$\\ 
		\midrule
		Random  &0.07&0.08&0.04&0.05&0.04&0.02&0.03&0.03\\
		NETTACK &0.61&0.57&0.55&0.53&0.75&0.71&0.66&0.61\\
		NETTACK - Sequential &0.68&0.73&0.72&0.70&0.76&0.74&0.72&0.67\\
		Fake Node Attack &0.61&0.58&0.54&0.52&0.76&0.68&0.62&0.60\\
		G-NIA &-&-&-&-&0.86&0.76&0.70&0.65\\
		%KDD Cup 1st Attack &0.61&0.55&0.51&0.42   &0.55&0.56&0.51&0.45\\
		%G-NIA &&&&&&&&\\
		\midrule
		%Cluster Attack - Naive&\textbf{1.00}&0.91&0.84&0.72&&\textbf{1.00}&0.91&0.82&0.73&\\
		Cluster Attack &\textbf{0.99}&\textbf{0.93}&\textbf{0.84}&\textbf{0.72}&\textbf{1.00}&\textbf{0.89}&\textbf{0.80}&\textbf{0.70}\\
		\bottomrule
	\end{tabular}
	%\vspace{-1ex}
	%}
	\caption{Success rates of Cluster Attack along with other baselines with discrete feature space. $T$ denotes number of victim nodes.
	%For $T=3$, we only add 3 fake nodes in our Cluster Attack.
	}
	%\vspace{-1ex}
	\label{exp-result}
\end{table*}

\paragraph{Parameters.} For each experimental setting, we run the experiment for 100 times and report the average results. In each round, we randomly sample $|\Phi_{A}|$ nodes as victim nodes. 
We set $k=1$ in $\mathcal{N}_{k}(\Phi_{A})$, which means we aim to protect the $1$-hop neighbors of victim nodes. Without specification, we compare our method with baselines with a trade-off parameter set as $\lambda=0$ in Eq.~\eqref{overall_goal}.

\paragraph{Comparison Methods.} Since this study is the first to perform query-based injection attack on node classification, most of the previous baselines on graph injection attacks cannot be easily adapted to our problem. We include the following baselines:%For instance, we don't include the existing RL-based method \cite{sun2020adversarial} as it's for attacking the nodes in the whole graph. We include the following baselines:

\textbf{Random Attack}, which decides the fake nodes' features and connections between fake nodes and original nodes randomly.

\textbf{NETTACK}, one of the most effective attacks by first adding several nodes and then adding many edges between the fake nodes and original nodes~\cite{Z_gner_2018}.

\textbf{NETTACK - Sequential}, which is a variant of NETTACK~\cite{Z_gner_2018} by sequentially adding fake nodes.%, instead of deciding all the connections and features one time. %Greedily make connections between the newly added fake node and original nodes and greedily decide the feature of the fake node every time a fake node is added.

\textbf{Fake Node Attack}, which adds fake nodes in a white-box attack scenario~\cite{wang2018attack}.%It decides the links between fake nodes and victim nodes and the feature of fake nodes using simple greedy method.

\textbf{G-NIA}, a white-box graph injection attack~\cite{gnia}. We refer to the reported results on Citeseer.

\textbf{TDGIA}, a black-box GIA method with superior performance to all the baselines in KDD Cup 2020\footnote{https://www.kdd.org/kdd2020/kdd-cup} of graph injection attack. We mainly compare our method with this method. Note that TDGIA is not query-based~\cite{tdgia}.

% We don't include NIPA \cite{sun2020adversarial} as our baseline. It is based on reinforcement learning and it's too time-consuming and impractical to train an RL-agent for every round of our attack. 
Among the above baselines, TDGIA is performed in a continuous feature space while NETTACK and Fake Node Attack are performed in discrete feature spaces.

\subsection{Quantitative Evaluation}
Without loss of generality, we uniformly set $N_{fake}=4$ and let the number of victim nodes vary to see the performance under different $N_{fake}:|\Phi_{A}|$.
\subsubsection{\textbf{Performance on Small Datasets with Discrete Features}}
% &0.34&0.45&0.55&0.66   &&&&&
We first evaluate the performance of the Cluster Attack along with other baselines on Cora and Citeseer with discrete features. The number of queries $K$ is set to $K=|\Phi_{A}|\cdot K_{t}+N_{fake}\cdot K_{f}$, where $K_{t}=K_{f}=D$ (feature dimension). The results are shown in Table \ref{exp-result}. Our algorithm outperforms all baselines in terms of success rates. This is because our method prevents inefficiently searching the non-Euclidean space of the adjacent matrix and better utilizes the limited queries in searching the Euclidean feature space. The results also demonstrate that the Adversarial Vulnerability is a good metric for clustering the victim nodes.
% The success rate of Cluster Attack gets higher when the number of victim nodes gets smaller with a fixed number of fake nodes. We conjecture that this is because the number of victim nodes in each cluster get smaller and thus the influence of the fake node on each victim node get relatively larger.

\subsubsection{\textbf{Performance on Large Datasets with Continuous Features}}
In this section, we evaluate the performance of the Cluster Attack on Reddit (with $1500{|\Phi_A|}+750N_{fake}$ queries) and obgn-arxiv (with $4000{|\Phi_A|}+2000N_{fake}$ queries), two large networks with continuous feature whose victim nodes have a higher average degree. We compare our method with the state-of-the-art method which has superior performance to other baselines. In these challenging datasets, we perform an untargeted attack, which means attacker successes when the predicted labels of victim nodes are changed. The results are shown in Table \ref{continuous-result}. Our algorithm outperforms the baseline in terms of success rates. This is because, with the cluster prior of the adjacent matrix, our Cluster Attack prevents inefficiently searching in non-Euclidean space and make the best use of the limited queries to search the Euclidean feature space. Another reason is because of the good metric of the Adversarial Vulnerability, which provides an appropriate cluster prior.
\begin{table}[tbhp]
    %\vspace{-2ex}

	\centering
	%\resizebox{\textwidth}{15mm}{
	\begin{tabular}{l|ll|ll}
		\toprule
		\multirow{2}{*}{Method} &
        \multicolumn{2}{c|}{ogbn-arxiv}&
        \multicolumn{2}{c}{Reddit} \\
        & $T=12$ & $T=16$ &$T=12$ & $T=16$\\ 
		\midrule
		TDGIA &0.45&0.38&0.09&0.07\\
		%G-NIA &&&&\\
		\midrule
		Cluster Attack &\textbf{0.67}&\textbf{0.59}&\textbf{0.15}&\textbf{0.12}\\%arxiv 4/10 0.74 reddit 4/10 0.15
		\bottomrule
	\end{tabular}
	%\vspace{-2ex}
	%}
	\caption{Success rates of Cluster Attack along with other baseline with continuous features. $T$ denotes number of victim nodes.}
	\label{continuous-result}
\end{table}

\subsection{Ablation Study}
\subsubsection{\textbf{Performance with Different Trade-Off Parameters $\lambda$}}
% \begin{table*}[tbhp]
% 	\caption{Success rates of Cluster Attack along with other baselines.}
% 	\label{exp-result3}
% 	\centering
% 	%\resizebox{\textwidth}{15mm}{
% 	\begin{tabular}{l|lllll|lllll}
% 		\toprule
% 		\multirow{2}{*}{Method} &
%         \multicolumn{5}{c|}{Cora}&
%         \multicolumn{5}{c}{Citeseer} \\
%         & $\lambda=0$ & $\lambda=0.1$ & $\lambda=1$ & $\lambda=10$ &$\lambda=100$ &$\lambda=0$ &$\lambda=0.1$ &$\lambda=1$&$\lambda=10$&$\lambda=100$ \\ 
% 		\midrule
% 		NETTACK - Sequential &0.62&0.48&0.51&0.39&0.34&&&&&\\
% 		NETTACK - Sequential - Neighbor Acc  &0.86&0.94&0.99&1.00&1.00&&&&&\\
% 		Fake Node Attack &0.52&0.39&0.38&0.41&0.41&&&&&\\
% 		Fake Node Attack - Neighbor Acc  &0.89&0.91&0.95&1.00&1.00&&&&&\\
% 		\midrule
% 		%Cluster Attack - Naive&\textbf{1.00}&0.91&0.84&0.72&&\textbf{1.00}&0.91&0.82&0.73&\\
% 		Cluster Attack &0.72&0.70&0.60&0.55&0.55&&&&&\\
% 		Cluster Attack - Neighbor Acc &0.90&0.94&0.98&1.00&1.00&&&&&\\
% 		\bottomrule
% 	\end{tabular}
% 	%}
% \end{table*}
In this section, we examine the performance of Cluster Attack with different trade-off parameters $\lambda$ between fake nodes and protected nodes in the Cora dataset. We uniformly set $N_{fake}=4$, $|\Phi_{A}|=10$. We choose two competitive baselines and adapt their loss functions to our trade-off format. The results are shown in Figure \ref{lambda}. It can be seen from Figure \ref{lambda} that, when $\lambda$ increases (which means that we pay more attention to the protected nodes), the percentage of protected nodes whose labels remain unchanged during the attack also increases. This is because we try to protect the labels of the protected nodes from being changed in a trade-off formulation in our loss function, Eq. (\ref{overall_goal}). Also, our trade-off formulation can be generalized to other baselines, as shown in Figure \ref{lambda}. This is because we design our loss function in Eq.~(\ref{overall_goal}) in a generalizable manner independent of attack method.
\begin{figure}[bhtp]
\centering
\subfigure[Success rates of attack]{
\begin{minipage}[t]{0.48\linewidth}
\centering
\includegraphics[width=1.5in]{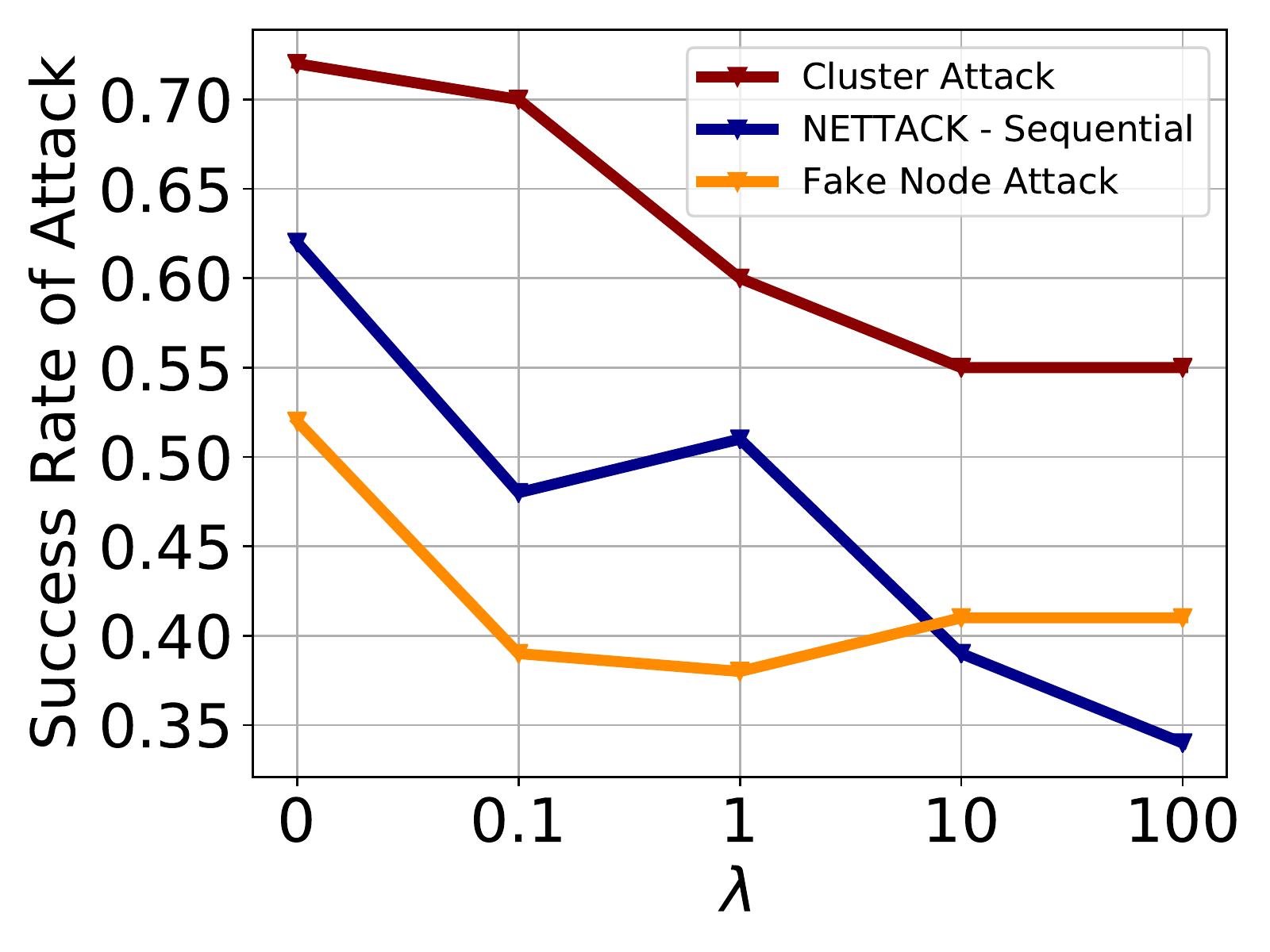}
\end{minipage}%
}%
\subfigure[Percentage of unchanged protected nodes
%whose Labels Remain Unchanged during Attack
]{
\begin{minipage}[t]{0.48\linewidth}
\centering
\includegraphics[width=1.5in]{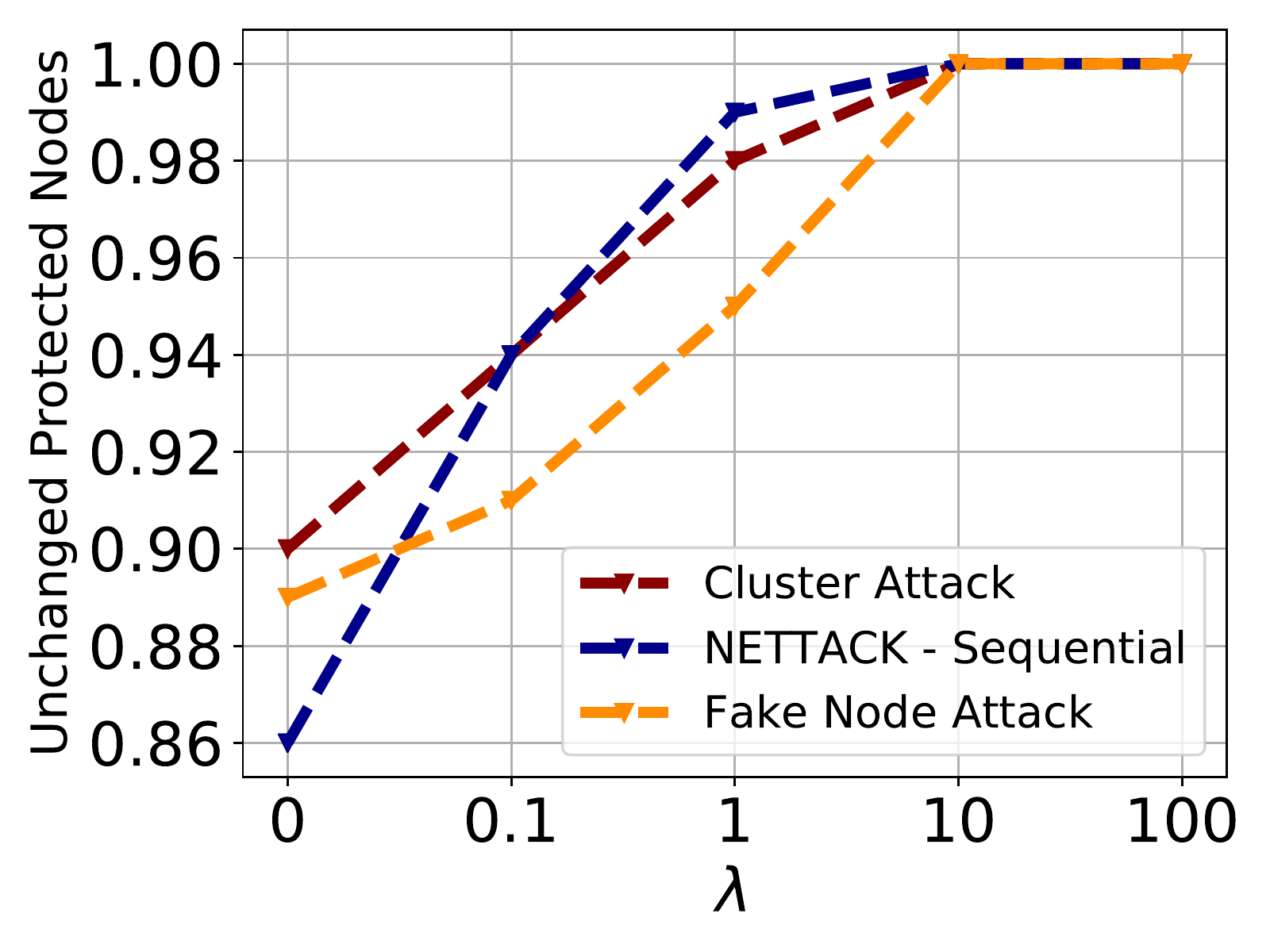}
\end{minipage}
}%
\centering
%\vspace{-2ex}
\caption{Cluster Attack in Cora with different $\lambda$.}
\label{lambda}
%\vspace{-2.2ex}
\end{figure}

\subsubsection{\textbf{Performance with Different Number of Queries}}
In this section, we examine the performance of Cluster Attack with a different number of queries. We set $K_{t}=K_{f}=\alpha \cdot D$ and examine the performance under different $\alpha$ in Cora and Citeseer dataset. We uniformly set $N_{fake}=4$, $|\Phi_{A}|=10$ with $\lambda=0$ and $\lambda=1$. The results are shown in Figure~\ref{query-result}. The success rate of Cluster Attack drops as the number of queries drops. Our algorithm still performs well when the number of queries drops slightly, especially when $\alpha\ge 0.4$. This demonstrates that our Cluster Attack can work in a query-efficient manner. This is because cluster algorithm provides graph-dependent priors for the adjacent matrix and thus prevents inefficient searching. Searching in the Euclidean feature space is more efficient.
%0.7 0.59  0.59 0.58 0.57 0.57 0.56 0.56 0.51 0.40
\begin{figure}[bhtp]
	\centering
	\includegraphics[width=1.8in]{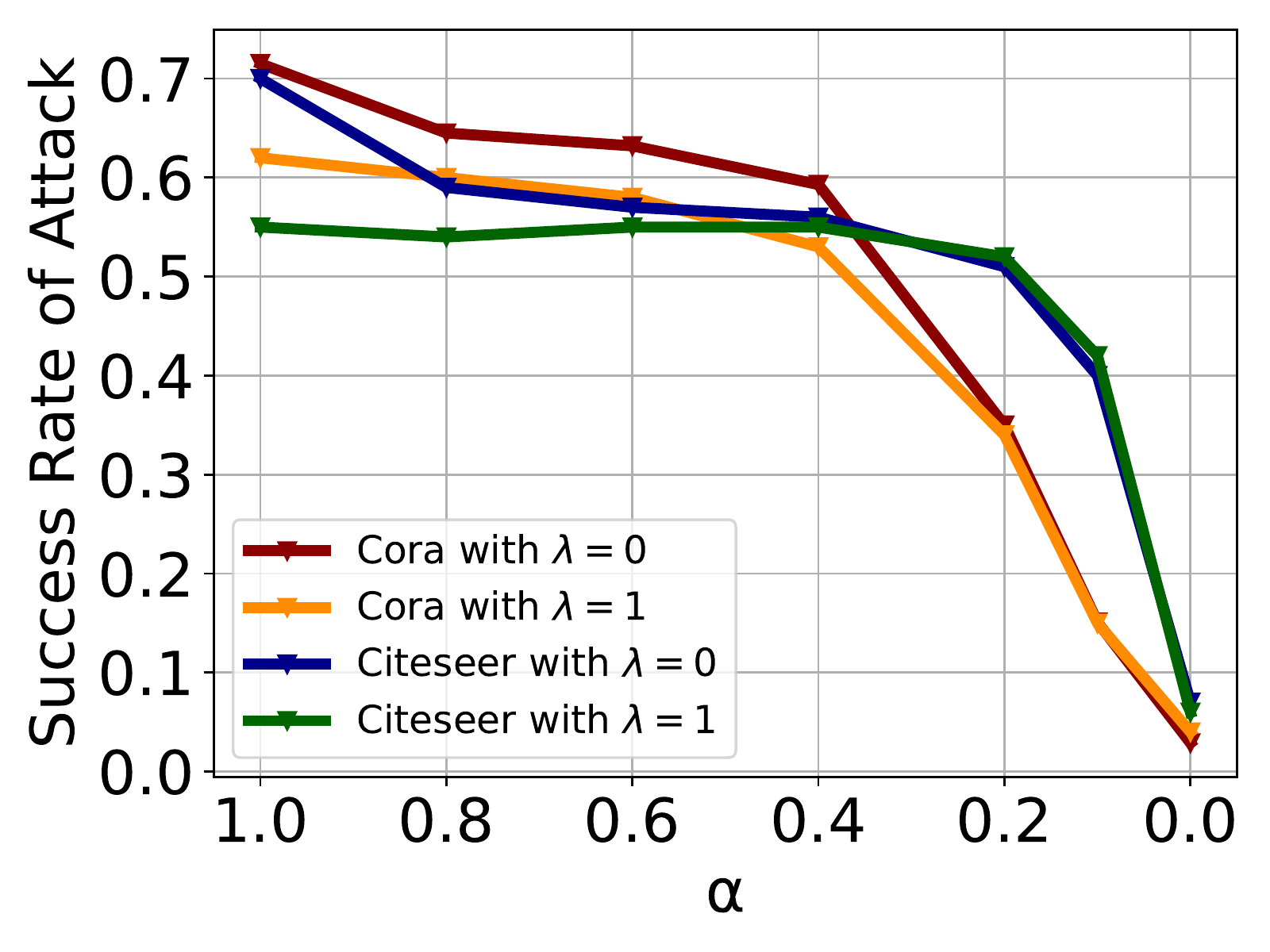}
	%\vspace{-1.3ex}
	\caption{Success rates of Cluster Attack with different number of queries in Cora and Citeseer.}
	\label{query-result}
	%\vspace{-1.7ex}
\end{figure}

% \subsubsection{\textbf{Effectiveness of Adversarial Vulnerability}}
% In this section, we provide an ablation study on the cluster metric. We examine the effectiveness of our Adversarial Vulnerability (AV). We uniformly set $N_{fake}=4$, $|\Phi_{A}|=10$. We compare the success rate of Cluster Attack without AV, i.e., the victim nodes' AVs are randomly set. The results are shown in Table \ref{random-result}. Cluster Attack without AV performs worse than original Cluster Attack with properly computed AV, which demonstrates the effectiveness of our cluster metric. AV is related to the vulnerability of victim nodes. Nodes with similar Adversarial Vulnerabilities in a cluster are easier to be affected together by one fake node. Thus the success rate of Cluster Attack with AV is better than without AV.
% \begin{table}[hbtp]
% 	\caption{Success Rates of Cluster Attack with and without AV in Cora dataset.}
% 	\label{random-result}
% 	\centering
% 	%\resizebox{\textwidth}{15mm}{
% 	\begin{tabular}{l|l}
% 		\toprule
%         Method& Success Rate\\ 
% 		\midrule
% 		Cluster Attack - without AV &0.62\\
% 		Cluster Attack&\textbf{0.72}\\
% 		\bottomrule
% 	\end{tabular}
% 	%}
% \end{table}

We provide additional experiments in the appendix. The experiments show that nodes with a lower degree are more likely to get misclassified under attack. Also, when the number of fake nodes increases, the success rate of the attack increases too, which is consistent with our intuitive understanding. We provide an ablation study on the cluster metric of Adversarial Vulnerability. We find that original Cluster Attack performs better than Cluster Attack without Adversarial Vulnerability, i.e., the victim nodes' Adversarial Vulnerabilities are randomly set. This result demonstrates the effectiveness of our Adversarial Vulnerability.
\section{Conclusion}
In this paper, we provide a unified framework for query-based adversarial attacks on graphs. Under the framework, we propose Cluster Attack, a query-based black-box graph injection attack with partial information. We demonstrate that a graph injection attack can be formulated as an equivalent clustering problem. %Moreover, we propose to attack by clustering the victim nodes according to the similarity in their adversarial vulnerability.
The difficult discrete optimization problem of the adjacent matrix can thus be solved in the context of clustering. After providing theoretical bounds on our method, we empirically show that our method has strong performance in terms of the success rate of attacking.

%\clearpage
\section*{Ethical Statement}
The safety and robustness of AI are attracting more and more attention. In this work, we propose a method of adversarial attack. We hope our work reveals the potential weakness of current graph neural networks to some extent, and more importantly inspires future work to develop more robust graph neural networks.
%\nocite{wu2019simplifying,DBLP:conf/iclr/XuHLJ19}
\section*{Acknowledgments}
This work was supported by the National Key Research and Development Program of China (Nos. 2020AAA0104304, 2017YFA0700904), NSFC Projects (Nos. 62061136001, 61621136008, 62076147, U19B2034, U19A2081, U1811461), the major key project of PCL (No. PCL2021A12), 
Tsinghua-Alibaba Joint Research Program,  Tsinghua-OPPO Joint Research Center, and the High Performance Computing Center, Tsinghua University.

\bibliography{example_paper}
\bibliographystyle{named}

\appendix
\clearpage

\section{Details of Cluster Attack}
We provide details on how to approximate Adversarial Vulnerability with limited queries for both discrete feature space (shown in Algorithm \ref{alg:fafam}) and continuous feature space (shown in Algorithm \ref{alg:fafam2}).
\begin{algorithm}[htb]
\caption{Approximation of Adversarial Vulnerability with Discrete Feature}
\label{alg:fafam}
\textbf{Input: }{Graph $G^{+}=(\mathbf{A}^{+},\mathbf{X}^{+})$. Victim node $v$. Number of queries $K_t$.}

\textbf{Output: }{Approximated Adversarial Vulnerability AV($v$) for $v$.}

\begin{algorithmic}[1]
\STATE {\bfseries initialize} Choose one fake node $u$ and connect it to and only to $v$, randomly initialize the fake node's feature $x_{u}$. Keep other fake nodes isolated.
\STATE Randomly sample a sequence $I_t$ from $\{1, 2, ..., D\}$ with length $K_t$. $D$ is the dimension of nodes' feature.
\FOR{$i \in I_t$}
\IF{$x_{u}[i]\leftarrow 1-x_{u}[i]$ makes 
$\mathcal{L}(G^{+})$ smaller}
\STATE $x_{u}[i]\leftarrow 1-x_{u}[i]$
\ENDIF
\ENDFOR
\STATE {\bfseries return} AV($v$)$\leftarrow x_{u}$
\end{algorithmic}
\end{algorithm}

\begin{algorithm}[htb]
\caption{Approximation of Adversarial Vulnerability with Continuous Feature}
\label{alg:fafam2}
\textbf{Input: }{Graph $G^{+}=(\mathbf{A}^{+},\mathbf{X}^{+})$. Victim node $v$. Number of queries $K_t=n*T$. $n$ is the size of NES population and $T$ is the number of iterations. Search variance $\sigma$. Optimization step size $\eta$.}

\textbf{Output: }{Approximated Adversarial Vulnerability AV($v$) for $v$.}

\begin{algorithmic}[1]
\STATE {\bfseries initialize} Choose one fake node $u$ and connect it to and only to $v$, randomly initialize the fake node's feature $x_{u}$. $l(x_u)$ is the corresponding loss. Keep other fake nodes isolated.
\FOR{$t = 1,2,...,T$}
\STATE $g \leftarrow 0$
\FOR{$i = 1,2,..., [\frac{n}{2}]$}
\STATE $u_i\leftarrow \mathcal{N}(0,I)$
\STATE $g \leftarrow g+l(x_u+\sigma * u_i)*u_i$
\STATE$g \leftarrow g-l(x_u-\sigma * u_i)*u_i$
\ENDFOR
\STATE $grad \leftarrow \frac{1}{n\sigma}g$
\STATE $x_u \leftarrow x_u - grad *\eta$
\STATE Clip $x_u$ between [$\min X, \max X$].
\ENDFOR
\STATE {\bfseries return} AV($v$)$\leftarrow x_{u}$
\end{algorithmic}
\end{algorithm}
The overall algorithm of Cluster Attack is shown in Algorithm \ref{alg:ffa}
\begin{algorithm}[htb]
\caption{Cluster Attack}
\label{alg:ffa}
%\hangx{not quite informative, and I think we may merge the information to figure 2, and put figure 2 here?}
% \hangx{could be more concise and focus on the key idea or main steps. Ideally, the reviewer can get your idea by reading your algorithm even they may missing some concepts in the main body. Therefore, we may remove the initialization and other unimportant information.}

\textbf{Input: }{Graph $G^{+}=(\mathbf{A}^{+},\mathbf{X}^{+})$. Victim node set $\Phi_{\mathbf{A}}$. Number of fake nodes $N_{fake}$. Number of queries $K=K_{t}*|\Phi_A|+K_{f}*|\Phi_{fake}|$.}

\textbf{Output: }{Manipulated graph $G^{+}=(\mathbf{A}^{+},\mathbf{X}^{+})$.}

\begin{algorithmic}[1]
\FORALL{$v \in \Phi_{\mathbf{A}}$}
%\STATE Connect one fake node $u$ to $v$ and temporarily disconnect $u$ with nodes other than $v$.
\STATE Compute AV($v$) with $K_{t}$ queries using Algorithm \ref{alg:fafam} or Algorithm \ref{alg:fafam2}.
\ENDFOR
\STATE Cluster the victim nodes according to their Adversarial Vulnerability and get equivalent adjacent matrix $\mathbf{B}$ using Proposition \ref{equivalence}.
\STATE Set $\mathbf{X}_{fake}$ as cluster center using Eq. (\ref{cluster_center}).
\FORALL{$v \in \Phi_{\mathbf{fake}}$}
\STATE Optimize the feature of fake node $v$ using Eq. (\ref{discrete}) and Eq. (\ref{continuous}) with $K_{f}$ queries.
\ENDFOR
% \FOR{$i = 1, 2, ..., N_{fake}$}
% \STATE Randomly sample a sequence $I_f$ from $\{1, 2, ..., D\}$ with length $K_{f}$. $D$ is the dimension of nodes' feature.
% \FOR{$j \in I_f$}
% \IF{$x_{v_{N+i}}[j]\leftarrow 1-x_{v_{N+i}}[j]$ makes 
% $\mathcal{L}(G^{+})$ smaller}
% \STATE $x_{v_{N+i}}[j]\leftarrow 1-x_{v_{N+i}}[j]$
% \ENDIF
% \ENDFOR
% \ENDFOR
\STATE {\bfseries return} $G^{+}=(\mathbf{A}^{+},\mathbf{X}^{+})$
\end{algorithmic}
\end{algorithm}

\section{Additional Experiments}
Here, we provide addition experiments on discrete feature space.
% We provide an extra baseline.
% \paragraph{KDD Cup 1st Attack} We adapt the method which won 1st place in KDD Cup 2020 competition on graph injection attack. Here, the feature in our experiments are discrete compared to its original continuous setting. We use greedy method to attack instead of the original gradient descent.
\subsection{\textbf{Performance with Different Number of Fake Nodes.}}
In this section, we evaluate the performance of Cluster Attack along with other baselines with different number of fake nodes. We fix the number of victim nodes and vary the number of fake nodes to examine the success rates. We set $|\Phi_{A}|=10$. The success rates are shown in Figure \ref{fakeN}. We see that the success rate is higher when there are more fake nodes. For Cluster Attack, we conjecture that this is because the number of victim nodes in each cluster gets smaller when there are more clusters. Thus the Adversarial Vulnerability of the victim nodes in each cluster are able to be closer to each other and they are easier to be attacked by the same fake node. Among all methods, our Cluster Attack achieves the highest success rate.
\begin{figure}[htbp]
\centering
\subfigure[Experiment on Cora]{
\begin{minipage}[t]{0.5\linewidth}
\centering
\includegraphics[width=1.6in]{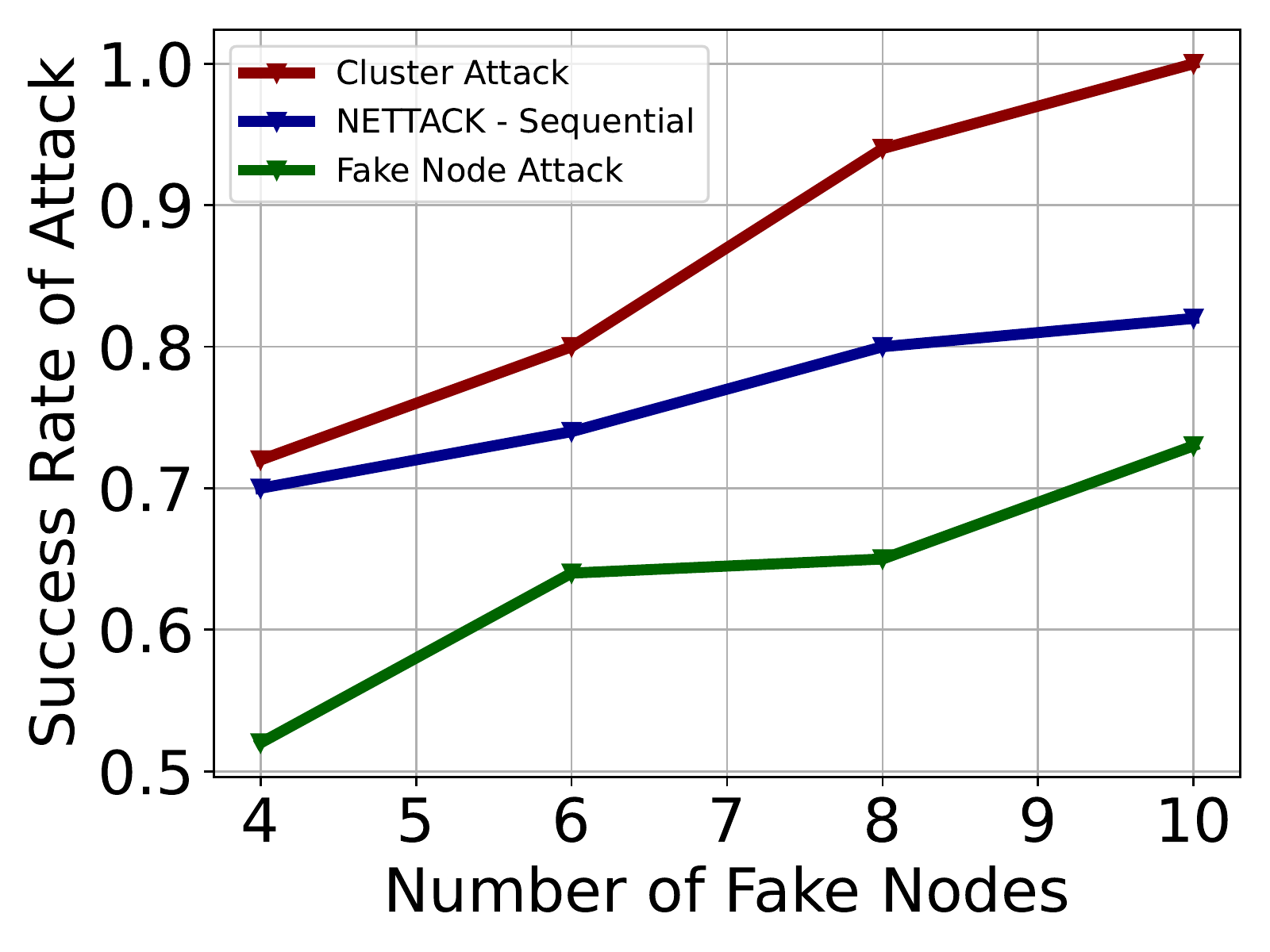}
\end{minipage}%
}%
\subfigure[Experiment on Citeseer]{
\begin{minipage}[t]{0.5\linewidth}
\centering
\includegraphics[width=1.6in]{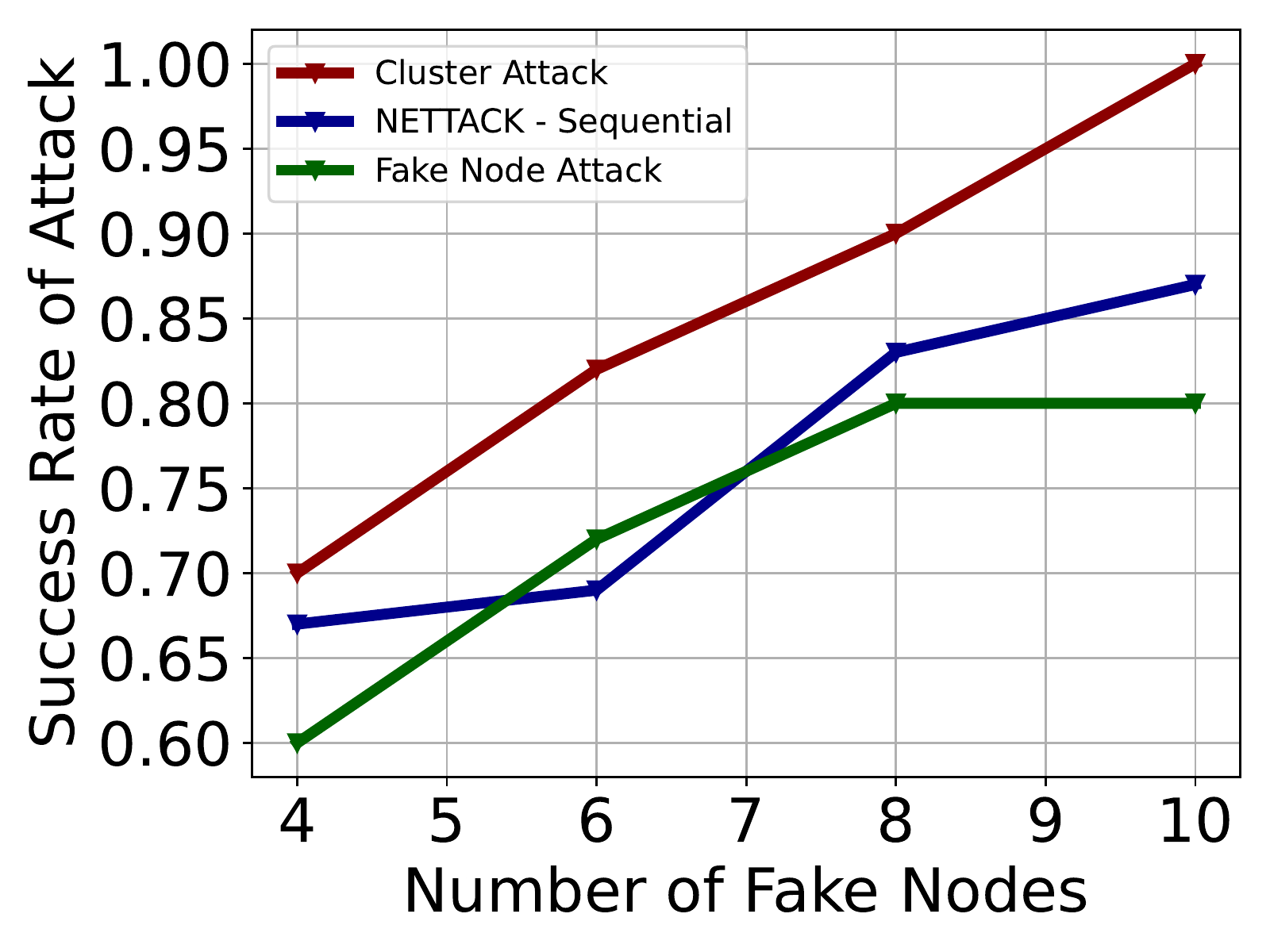}
\end{minipage}
}%
\centering
\caption{Success rates of Cluster Attack with different number of fake nodes along with other baselines.}
\label{fakeN}
\end{figure}
% \begin{table*}[tbhp]
% 	\caption{Success rates of Cluster Attack along with other baselines. $F$ denotes number of fake nodes.}
% 	\label{exp-result2}
% 	\centering
% 	%\resizebox{\textwidth}{15mm}{
% 	\begin{tabular}{l|llll|llll}
% 		\toprule
% 		\multirow{2}{*}{Method} &
%         \multicolumn{4}{c|}{Cora}&
%         \multicolumn{4}{c}{Citeseer} \\
%         & $F=4$ & $F=6$ & $F=8$ & $F=10$ &$F=4$ & $F=6$ & $F=8$ & $F=10$\\ 
% 		\midrule
% 		%NETTACK &&&&&&&&\\
% 		NETTACK - Sequential &0.70&0.74&0.80&0.82&0.67&0.69&0.83&0.87\\
% 		%Fake Node Attack &0.62&0.63&0.68&0.73&&&&\\ attacked node = 4
% 		Fake Node Attack &0.52&0.64&0.65&0.73&0.60&0.72&0.80&0.80\\
% 		KDD Cup 1st Attack &0.42&0.54&0.55&0.68&0.45&0.55&0.67&0.75\\
% 		\midrule
% 		Cluster Attack &\textbf{0.72}%&0.66
% 		&\textbf{0.80}&\textbf{0.94}&\textbf{1.00}
% 		&\textbf{0.70}%&0.55
% 		&\textbf{0.82}&\textbf{0.90}&\textbf{1.00}\\
% 		\bottomrule
% 	\end{tabular}
% 	%}
% \end{table*}
% Kddcup 0.61 0.59 0.69 

\subsection{\textbf{Analysis of Cluster Attack on Victim Nodes with Different Degrees.}}
In this section, we evaluate the performance of Cluster Attack on victim nodes with different degrees. We uniformly set $N_{fake}=4$, $\Delta_{edge}=|\Phi_{A}|=10$. The success rates of Cluster Attack on victim nodes with different degrees are shown in Figure \ref{degree} along with the proportion of sampled victim nodes with each degree. Victim nodes with degrees larger than or equal to 7 are counted together since they only account for a small proportion.
\begin{figure}[htbp]
\centering
\subfigure[Experiment on Cora]{
\begin{minipage}[t]{0.5\linewidth}
\centering
\includegraphics[width=1.6in]{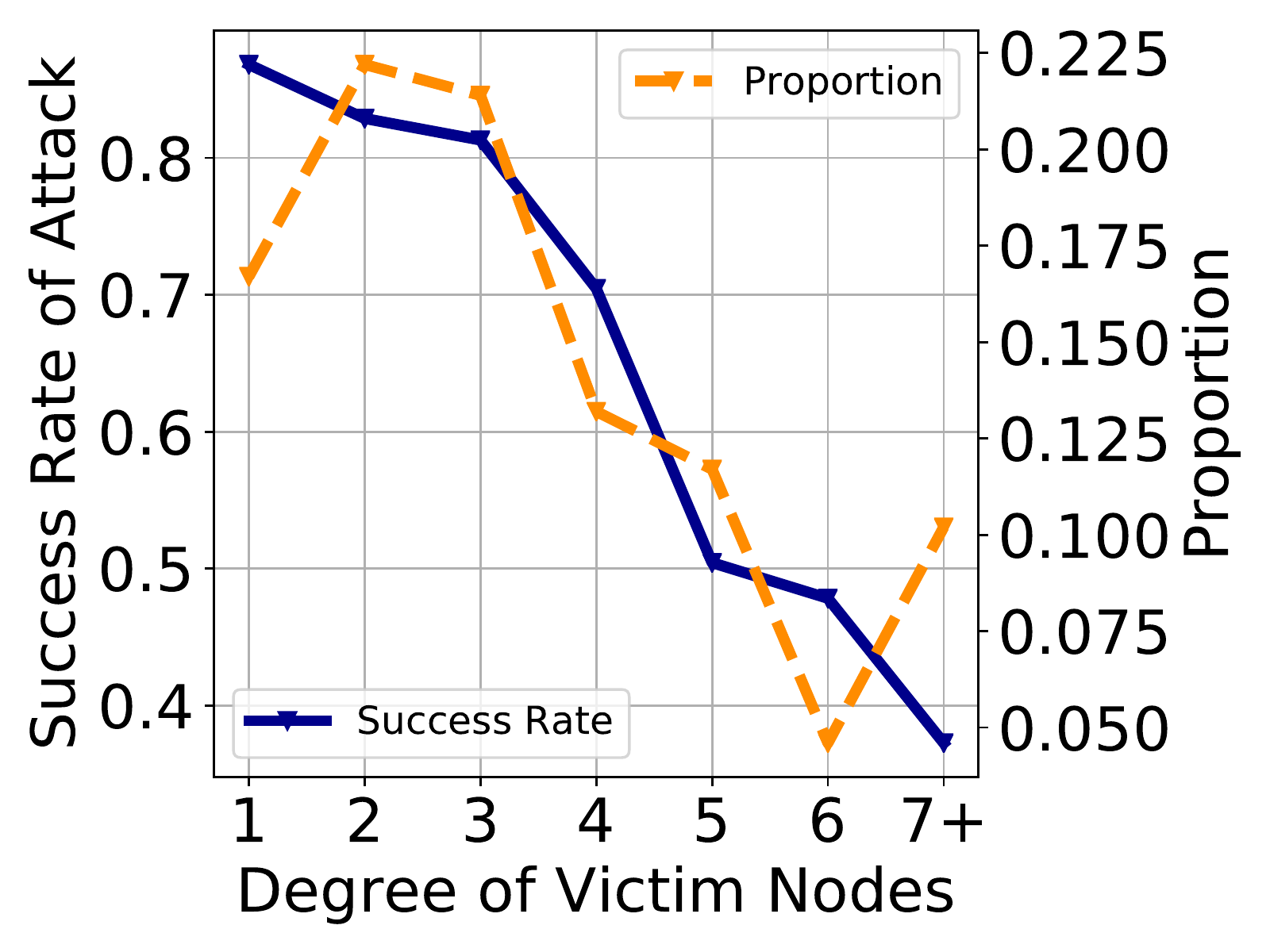}
\end{minipage}%
}%
\subfigure[Experiment on Citeseer]{
\begin{minipage}[t]{0.5\linewidth}
\centering
\includegraphics[width=1.6in]{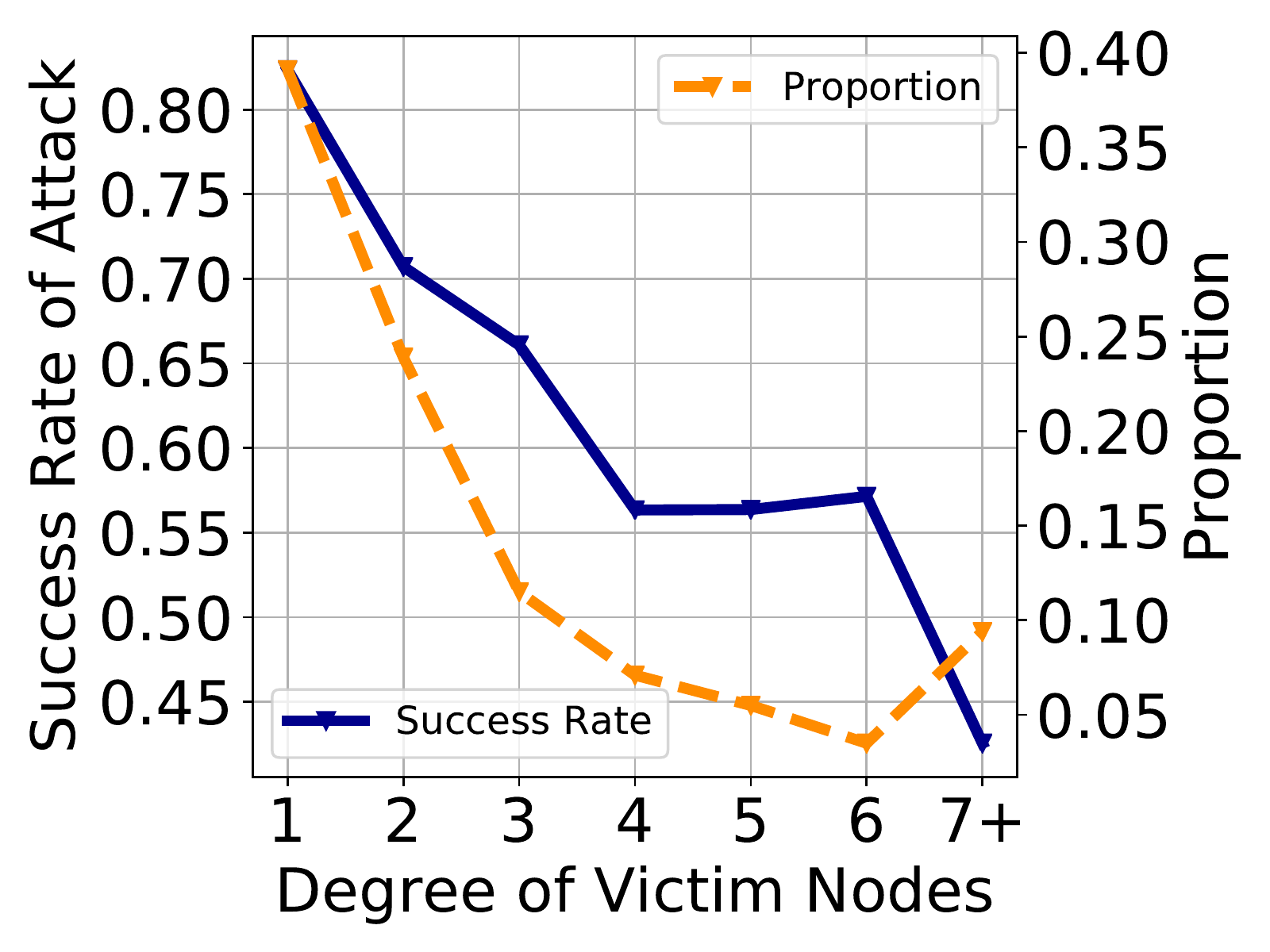}
\end{minipage}
}%
\centering
\caption{Success Rates of Cluster Attack on Victim Nodes with Different Degrees.}
\label{degree}
\end{figure}

It can be seen from Figure \ref{degree} that victim nodes with higher degrees are more robust to our attack in general. We conjecture that when a victim node has a relatively large number of neighbors, adding one fake node as its neighbor has less impact on it and thus is less likely to change its predicted label.

\subsection{\textbf{Ablation Study.}}
In this section, we examine the effectiveness of our cluster metric Adversarial Vulnerability (AV). We uniformly set $N_{fake}=4$, $\Delta_{edge}=|\Phi_{A}|=10$. We examine the success rate of Cluster Attack without AV, i.e., the victim nodes' AVs are randomly set. The results are shown in Table \ref{random-result}. Cluster Attack without AV performs worse than vanilla Cluster Attack with AV, which demonstrates the effectiveness of our AV. AV is related to the vulnerability of victim nodes. Nodes with similar AVs in a cluster are easier to be affected together by single fake node. Thus the success rate of Cluster Attack with AV is better than without AV.
\begin{table}[hbtp]
	\caption{Success rates of Cluster Attack with and without AV on Cora.}
	\label{random-result}
	\centering
	%\resizebox{\textwidth}{15mm}{
	\begin{tabular}{l|l}
		\toprule
        Method& Success Rate\\ 
		\midrule
		Cluster Attack - without AV &0.62\\
		Cluster Attack&\textbf{0.72}\\
		\bottomrule
	\end{tabular}
	%}
\end{table}

\end{document}